\documentclass{article}
\usepackage[utf8]{inputenc}

\usepackage[english]{babel}
\usepackage[T1]{fontenc}
\usepackage[margin=1in]{geometry} 
\usepackage[colorlinks=true, allcolors=blue]{hyperref}

\usepackage{nicefrac}
\usepackage{tikz}
\usepackage{amsmath}
\usepackage{amsthm}
\usepackage{amssymb}
\usepackage{mathtools}
\usepackage{graphicx}
\usepackage{subcaption}
\usepackage{amssymb}
\usepackage{bm}
\usepackage{xspace}
\usepackage{xcolor}
\usepackage{comment}
\usepackage{float} 
\usepackage{thmtools}
\usepackage{thm-restate}
\usepackage{cleveref}

\newtheorem{theorem}{Theorem}[section]
\newtheorem{lemma}[theorem]{Lemma}
\newtheorem{definition}[theorem]{Definition}

\newtheorem{fact}[theorem]{Fact}

\newtheorem*{theorem*}{Theorem}
\newtheorem*{problem*}{Problem}

\DeclareMathOperator{\A}{\mathcal{A}\xspace}
\newcommand{\f}{p\xspace}
\newcommand{\RS}{RB\xspace}

\DeclareMathOperator{\D}{\mathcal{D}}
\DeclareMathOperator{\I}{\mathcal{I}}

\DeclareMathOperator{\x}{\mathbf{x}\xspace}
\DeclareMathOperator{\y}{\mathbf{y}\xspace}

\DeclareMathOperator{\z}{\mathbf{z}\xspace}

\DeclareMathOperator{\w}{\mathbf{w}\xspace}

\DeclareMathOperator{\opt}{\mathrm{OPT}\xspace}

\DeclareMathOperator{\lp}{\mathrm{LP}\xspace}
\DeclareMathOperator{\supp}{\mathcal{A}\xspace}

\newcommand\Ex[2]{\mathop{\underset{#1}{\mathbb{E}}\left[#2\right]}}

\newcommand\Pro[1]{\mathop{\mathbb{P}\left[#1\right]}}
\newcommand\Prob[2]{\mathop{\underset{#1}{\mathbb{P}}\left[#2\right]}}

\usepackage[algo2e,ruled,noend]{algorithm2e} 

\SetCommentSty{mycommfont}

\crefname{algocfline}{Algorithm}{Algorithms}
\Crefname{algocfline}{Algorithm}{Algorithms}

\title{Non-Stationary Bandits under Recharging Payoffs:\\ Improved Planning with Sublinear Regret}

\author{Orestis Papadigenopoulos\\ Department of Computer Science \\
The University of Texas at Austin \\ ~~~~~~~~~~~~~~~~~~\texttt{papadig@cs.utexas.edu}~~~~~~~~~~~~~~~~~~  
\and Constantine Caramanis \\ Electrical and Computer Engineering \\ The University of Texas at Austin \\ \texttt{constantine@utexas.edu}
\and Sanjay Shakkottai \\ Electrical and Computer Engineering \\ The University of Texas at Austin \\ \texttt{sanjay.shakkottai@utexas.edu}
}

\date{}

\begin{document}

\maketitle

\begin{abstract}
The stochastic multi-armed bandit setting has been recently studied in the non-stationary regime, where the mean payoff of each action is a non-decreasing function of the number of rounds passed since it was last played. This model captures natural behavioral aspects of the users which crucially determine the performance of recommendation platforms, ad placement systems, and more. Even assuming prior knowledge of the mean payoff functions, computing an optimal planning in the above model is NP-hard, while the state-of-the-art is a $1/4$-approximation algorithm for the case where at most one arm can be played per round.
We first focus on the setting where the mean payoff functions are known. In this setting, we significantly improve the best-known guarantees for the planning problem by developing a polynomial-time $(1-{1}/{e})$-approximation algorithm (asymptotically and in expectation), based on a novel combination of randomized LP rounding and a time-correlated (interleaved) scheduling method. Furthermore, our algorithm achieves improved guarantees -- compared to prior work -- for the case where more than one arm can be played at each round. Moving to the bandit setting, when the mean payoff functions are initially unknown, we show how our algorithm can be transformed into a bandit algorithm with sublinear regret. 
\end{abstract}

\section{Introduction}
In the last two decades, the predominant rise of the social media industry has made the notion of a {\em newsfeed} an integral part of our lives. In a newsfeed, a user observes a structured sequence of content items (posts, photos etc.) particularly selected by the platform according to her/his preferences. Apart from social media, an analogous idea -- potentially relabeled -- also appears in different domains as, for example, ``frequently bought together'' in e-commerce, ``shuffling similar songs'' in music recommendation, or ``recommended articles'' in scholarly literature indexing databases. 
Whether it is measured in terms of click-rate or time devoted, the high-level objective of newsfeeds is fairly well-known: to maximize the user's engagement with the platform. In many applications, however, achieving this objective is not as simple as identifying the user's ``favorite'' content, given that her/his satisfaction can depend on the time passed since the same (or similar) content has been observed. As an example, a user's engagement can worsen if a social media feed (resp., a music recommendation platform) constantly presents content from the same source (resp., same artist). 

Motivated by such scenarios, researchers have recently studied online decision making problems capturing the notion of ``recovering'' payoffs, namely, scenarios where the payoff of an action drops (to zero) after each play and then slowly increases back to a baseline. In the context of online learning, these non-stationary models interpolate between multi-armed bandits, where the environment is assumed to be intact, and reinforcement learning, since the actions may now alter the future environment in a structured manner. 
These models are wide enough to capture many real-life applications, yet special enough to accept efficiently computable (near-)optimal solutions. In this direction, the following general model was first introduced (under slightly different assumptions) by Immorlica and Kleinberg \cite{KI18} and has been recently studied by Simchi-Levi et al. \cite{SLZZ21}: 

\begin{problem*}[Multi-armed bandits under recharging payoffs]
We consider a set of {\em actions} (or {\em arms}), where each arm $i$ is associated with a {\em mean payoff function} $\f_i$. For each $\tau$, $\f_i(\tau)$ is the expected payoff collected for playing action $i$, when $i$ has been played before exactly $\tau$ rounds (we refer to $\tau$ as the {\em delay} of an arm at a specific round). For each arm, we assume that its payoff function: (a) is monotone non-decreasing in $\tau$, and (b) has a known finite {\em recovery time}, namely, a specific delay after which the function remains constant. 
At each round, the player pulls a subset of at most $k$ actions, observes the realized payoffs of the arms played (each computed with respect to the delay of the corresponding action), and collects their sum. The objective is to maximize the expected cumulative payoff collected within an unknown time horizon.
\end{problem*}

We denote by $k$-\RS an instance of the problem where at most $k$ arms can be played at each round. Further, we distinguish between the {\em planning} setting, where the payoff functions are known to the player a priori (thus, payoffs can be considered deterministic), and the {\em learning} setting, where these are initially unknown and the user assumes semi-bandit feedback on the payoffs of the played actions. In the former case, the goal is to design an efficient algorithm which closely-approximates the optimal planning (since this is generally an NP-hard problem), while in the latter, the objective is to construct a bandit algorithm of sublinear regret, defined as the difference in expected payoff between a planning algorithm and its bandit variant due to the initial absence of information of the latter.

\subsection{Background and Related Work}

Immorlica and Kleinberg \cite{KI18} first study the $1$-\RS problem and provide a $(1-\epsilon)$-approximation (asymptotically) for the planning setting, under the additional assumption that the payoff functions are weakly concave over the whole time horizon. On the other extreme, Basu et al.\ \cite{BSSS19} provide a long-run $\left(1 - \nicefrac{1}{e}\right)$-approximation for the case where, after each play, an arm {\em cannot} be played again for a fixed number of rounds -- a problem which can be cast as a special instance of $1$-\RS using Heaviside step payoff functions.\footnote{Recall that the Heaviside step function is defined such that $H(\tau) = 1$ for $\tau \geq 0$, and $H(\tau) = 0$, otherwise.} Shortly after the introduction of the problem, a number of special cases and variations have been studied \cite{CCB19,PC21,PBG19,BPCS20,BCMT20,CCB19,LCCBGB21,APBCS21} (we address the reader to \Cref{sec:relatedwork} for an overview). 

In their recent work, Simchi-Levi et al.\ \cite{SLZZ21} study the $k$-\RS problem and prove the first $\mathcal{O}(1)$-approximation guarantee under no additional assumptions. For large $k$, this result is the state-of-the-art. However, for small $k$, understanding the magnitude of the constant becomes the primary theoretical question. Specifically, the motivating case of newsfeeds (when, for instance, the content items are presented to a user {sequentially}) can be modeled as an instance of $1$-\RS. In this case, the approximation guarantee of \cite{SLZZ21} becomes $\nicefrac{1}{4}$, which is significantly weaker compared to both the $(1-\epsilon)$-approximation of \cite{KI18} for concave functions and the $\left(1 - \nicefrac{1}{e} \right)$-approximation of \cite{BSSS19} for the most ``extreme'' example of monotone convex functions, i.e., that of Heaviside step functions. The above observation indicates that either the approximability status of the problem is not well-understood, or that the problem does not gradually become ``easier'' by increasing the concavity of the payoff functions.

\subsection{Technical Contributions}
In our work, we resolve this discrepancy by designing a long-run $\left(1 - \nicefrac{1}{e}\right)$-approximation algorithm (in expectation) for $1$-\RS, which improves the state-of-the-art due to \cite{SLZZ21}. Simultaneously, our algorithm enjoys the same asymptotic guarantee of $\left(1 - \mathcal{O}(\nicefrac{1}{\sqrt{k}})\right)$ for the general case of $k$-\RS as in Simchi-Levi et al.\ \cite{SLZZ21} with improved and explicit constants, as opposed to the guarantees of \cite{SLZZ21} which do not come in a closed-form. Our algorithm is based on a novel combination of linear programming (LP) rounding and a time-correlated (interleaved) scheduling method, and is significantly simpler to implement compared to prior work. For the case where the mean payoff functions are initially unknown, we show how our algorithm can be transformed into a bandit algorithm with sublinear regret guarantees. 

\subsection{Key Ideas and Intuition}

Henceforth, we refer to any instance where the payoff function of each action $i$ has the form $q_i(\tau) = p_i \cdot H(\tau - d_i)$, where $p_i \in [0,1]$, $d_i \in \mathbb{N}$, and $H(\cdot)$ the Heaviside step function, as the {\em Heaviside $k$-\RS} problem. The fact that our $\left(1 - \nicefrac{1}{e}\right)$-approximation for $1$-\RS matches the state-of-the-art bound obtained for Heaviside $1$-\RS \cite{BSSS19,PC21} is no coincidence. 
Our solution technique shows that each $k$-\RS problem is in fact (approximately) hiding a Heaviside $k$-\RS problem.

\paragraph{Structural characterization of a natural LP relaxation.} 
A key idea in both \cite{KI18} and \ \cite{SLZZ21} is to construct a concave relaxation of the optimal solution. Instead, we take a more direct approach and construct a natural LP relaxation (see \Cref{sec:relaxation}) of the optimal average payoff (our construction uses a uniform upper bound on the recovery time of any arm). 
By carefully analyzing the structure of our LP, we prove that its extreme point solutions follow a particularly interesting {\em sparsity pattern}: there exist unique delays \{$\tau_i$\} associated with the arms (which we refer to as ``critical'' delays), such that playing each arm $i$ exactly once every $\tau_i$ rounds matches the average payoff of the relaxation. As we show, an exception to the above rule can be at most a single arm, which we refer to as {\em irregular}. The above observation already hints that our problem could potentially be reduced to an instance of Heaviside $k$-\RS for which better approximation guarantees (compared to \cite{SLZZ21}) are known to exist (at least for small $k$) \cite{BSSS19,PC21}.

\paragraph{Improved approximation guarantees for the planning $k$-\RS.} 
Constructing a planning where each arm is played at a rate indicated by its critical delay (as described above) is generally infeasible, since these rates stem from a relaxation of the problem. However, targeting these rates becomes the starting point of our approximation. 
In \Cref{sec:algorithm}, we design a novel algorithm for the $k$-\RS problem, called ``Randomize-Then-Interleave'', which starts from computing an optimal solution to our LP relaxation. By combining the sparse structure of this solution with a {\em randomized rounding} step for determining a critical delay for the (possible) irregular arm, our algorithm produces a ``proxy'' instance of the Heaviside $k$-\RS problem. Finally, the algorithm applies the {\em interleaved scheduling} method of \cite{PC21} on the resulting Heaviside $k$-\RS instance. As we prove, the above algorithm achieves a long-run $\left(1 - \mathcal{O}(\nicefrac{1}{\sqrt{k}})\right)$-approximation guarantee (in expectation), but with improved and explicit constants compared to \cite{SLZZ21}. Importantly, for the particular case of $k$-\RS, our guarantee is stronger than the one provided in the original work of \cite{PC21}. 

As a first step, in \Cref{sec:easy} we study the performance of our algorithm under the simplifying assumption that the solution returned by our LP relaxation does not contain an irregular arm, that is, every arm is associated with a unique critical delay. For this case, the approximation guarantee of our algorithm follows by critically leveraging existing results on the {\em correlation gap} of the weighted rank function of uniform matroids due to \cite{yan11} (see \Cref{sec:definition}). In \Cref{sec:irregular}, we relax this simplifying assumption by carefully studying the contribution of the irregular arm to the produced solution. Our strategy relies on the novel idea of considering parallel (fictitious) copies of the irregular arm, one for each possible critical delay. While the marginal probabilities of each copy being played under the corresponding delay at any round are consistent with the initial LP solution, these events are {\em mutually exclusive}. Nevertheless, we show that expected payoff collected can only decrease by assuming that the parallel copies were instead independent. Having established that, the analysis reduces back to the ``easier'' case where an irregular arm does not exist. 

\paragraph{A bandit adaptation with sublinear regret.} Finally, we turn our focus to the learning setting, where the mean payoff functions are initially unknown, and the player has semi-bandit feedback on the realized payoffs of the played arms of each round. In this setting, we begin by analyzing the robustness of our planning algorithm under small perturbations of the payoff functions. Then, we provide sample complexity results to bound the number of samples required to get accurate estimates of the mean payoff functions. By combining the above elements, we transform our planning algorithm into a bandit one, based on an Explore-then-Commit (ETC) scheme, and prove that the latter achieves sublinear regret relative to its planning counterpart. We present these results in \Cref{sec:learning}. 

All the omitted proofs of our results have been moved to the Appendix.

\section{Preliminaries}\label{sec:definition}

\paragraph{Problem definition and notation.} We consider a set $\A = \{1,2, \ldots, n\}$ of $n$ {\em actions} (or {\em arms}), where each arm $i \in \A$ is associated with a (mean) {\em payoff function} $\f_i: \mathbb{N} \rightarrow [0,1]$. For each $\tau \in \mathbb{N}$, $\f_i(\tau)$ is the expected payoff for playing action $i \in \A$, when $i$ has been played before exactly $\tau$ rounds (i.e., $\tau = 1$ if the action has been played in the previous round). We refer to the number of rounds since the last play of an arm as the {\em delay}. For each arm $i \in \A$, we assume that the function $\f_i(\tau)$: (a) is monotone non-decreasing in $\tau$, and (b) has a polynomial {\em recovery time} $\tau^{\max}_i$ such that $\f_i(\tau) = \f_i(\tau^{\max}_i)$ for every $\tau \geq \tau^{\max}_i$. We assume knowledge of a universal upper bound $\tau^{\max}$ on the recovery time of all arms, i.e., $\tau^{\max} \geq \max_{i \in \A} \tau^{\max}_i$.\footnote{Due to the polynomial recovery time assumption, oracle access to the payoff functions is not required.} At each round, the player plays a subset of at most $k<n$ actions and collects the sum of their associated payoffs (each computed with respect to the delay of the corresponding action at the time it is played). The objective is to maximize the cumulative payoff collected within an (unknown) horizon of $T$ rounds. We refer to the above instance as $k$-\RS. 

As a convention, we assume that the delay of each arm $i$ is initially equal to $1$ (i.e., all arms are played at time $t=0$). Note that, even if we assume that all arms start with delay $\tau^{\max}$ (i.e., the arms have never been played before), our results do not change qualitatively. 

For any non-negative integer $n \in \mathbb{N}$, we define $[n] = \{1,2, \dots, n\}$. For any vector $\w \in \mathbb{R}^n$ and set $S \subseteq [n]$, we denote $\w(S) = \sum_{i \in S} w_i$. Finally, throughout this work we assume that, when comparing between different payoffs, ties are broken arbitrarily.

\paragraph{Continuous extensions and the correlation gap.}
Let $f : 2^{[n]} \rightarrow [0,\infty)$ be a set function defined over a finite set of $n$ elements. For any vector $\y \in [0,1]^n$, we denote by $\D(\y)$ a distribution over $2^{[n]}$ with marginal probabilities $\y = (y_i)_{i \in [n]}$. We denote by $S \sim \D(\y)$ a random subset $S \subseteq [n]$ sampled from the distribution $\D({\y})$. Further, we denote by $S \sim \I(\y)$ a random subset $S \subseteq [n]$ which is constructed by adding each element $i \in [n]$ to $S$, independently, with probability $y_i$. 

We recall two canonical continuous extensions of a set function (see \cite{calinescu2007maximizing,schrijver03}): 

\begin{definition}[Multi-linear extension] \label{def:multilinear}
For any vector $\y \in [0,1]^n$, the {\em multi-linear extension} of a set function $f$ is defined as
\begin{align*}
    F(\y) = \Ex{S \sim \I(\y)}{f(S)} = \sum_{S \subseteq [n]} f(S) \prod_{i \in S} y_i \prod_{i \notin S} (1-y_i).
\end{align*}

\end{definition}

\begin{definition}[Concave closure] \label{def:concaveclosure}
For any vector $\y \in [0,1]^n$, the {\em concave closure} of a set function $f$ is defined as
\begin{align*}
    f^+(\y) = \sup_{\D(\y)} \Ex{S \sim \D(\y)}{f(S)} = \sup_{\boldsymbol \alpha} \bigg\{ \sum_{S \subseteq [n]} \alpha_S f(S) \mid \sum_{S \subseteq [n]} \alpha_S \bm{1}_S = \y, \sum_{S \subseteq [n]} \alpha_S = 1, {\boldsymbol \alpha} \succeq 0 \bigg\},
\end{align*}
where $\bm{1}_S \in \{0,1\}^n$ is an indicator vector such that $(\bm{1}_S)_i = 1$, if $i \in S$, and $(\bm{1}_S)_i = 0$, otherwise.
\end{definition}

For any non-negative weight vector $\w \in [0, \infty)^n$, of particular importance in the analysis of our algorithm is the function $f_{\w,k} : 2^{[n]} \rightarrow [0, \infty)$, defined as
\begin{align}
f_{\w,k}(S) = \max\{\w(I) ~\mid~ I \subseteq S \text{ and } |I| \leq k\}.    \label{eq:uniform}
\end{align}
We remark that $f_{\w,k}$ corresponds to the {\em weighted rank function} of the (rank-$k$) uniform matroid under a weight vector $\w$ and, hence, is non-decreasing submodular \cite{schrijver03}.\footnote{A function $f: 2^{[n]} \rightarrow \mathbb{R}$ is submodular, if for any $S,T \subseteq [n]$, it holds $f(S \cup T) + f(S \cap T) \leq f(S) + f(T)$.} We denote by $F_{\w,k}$ and $f^+_{\w,k}$ the multi-linear extension and the concave closure of $f_{\w,k}$, respectively. 

The {\em correlation gap} \cite{ADSY10} of a set function is defined as $\sup_{\y \in [0,1]^n} \frac{f^+(\y)}{F(\y)}$. 
The following result due to \cite{yan11} provides an upper bound on the correlation gap of $f_{\w,k}(\cdot)$:

\begin{lemma}[Correlation gap \cite{yan11}] \label{lem:correlationgap}
Let $f_{\w,k}: 2^{[n]} \rightarrow [0, \infty)$ be the {\em weighted rank function} of the rank-$k$ uniform matroid. Then, for any vector $\y \in [0,1]^n$ we have
$$
f^+_{\w,k}(\y) \geq F_{\w,k}(\y) \geq {\left(1 - \frac{k^k}{e^k k!} \right)} \cdot f^+_{\w,k}(\y) \approx {\left(1 - \frac{1}{\sqrt{2\pi k}} \right)} \cdot f^+_{\w,k}(\y).
$$
\end{lemma}

\section{Structural Properties of a Natural LP Relaxation} \label{sec:relaxation}
We consider the following LP relaxation which applies to any planning instance of $k$-\RS:
\begin{align}
\underset{{\x \succeq \bf 0}}{\textbf{maximize: }}&~ \sum_{i \in \A} \sum_{\tau \in \mathbb{N}} p_i(\tau) \cdot x_{i,\tau}  \label{lp:LP} \tag{\textbf{LP}$_k$}\\
\textbf{s.t.: }& \sum_{i \in \A} \sum_{\tau \in \mathbb{N}} x_{i,\tau} \leq k, \label{lp:total}\tag{C.1}\\
&~\sum_{\tau \in \mathbb{N}} \tau \cdot x_{i,\tau} \leq 1, \forall i \in \A. \label{lp:arm} \tag{C.2}
\end{align}
In the above formulation, each variable $x_{i,\tau}$ denotes the fraction of time where arm $i$ is played under delay $\tau$ -- namely, exactly $\tau$ time steps after it was played for the last time. In \eqref{lp:LP}, constraint \eqref{lp:total} follows by the fact that the total fraction of time where any action can be played (under any delay) cannot be more than $k$. Further, constraints \eqref{lp:arm} are valid constraints for the fractional plays of any individual arm under different delays. 

We start by proving that the optimal solution of \eqref{lp:LP} is asymptotically an upper bound on the optimal average expected payoff: 

\begin{restatable}{lemma}{restateUpperBound} \label{lem:lpupperbound}
For any instance of the $k$-\RS problem, let $V^*$ be the optimal value of \eqref{lp:LP} and $\opt(T)$ be the maximum payoff that can be collected in a time horizon of $T$ rounds. Then $T \cdot V^* \geq \opt(T).$
\end{restatable}

Assuming knowledge of an upper bound $\tau^{\max} \geq \max_{i \in \mathcal{A}} \{\tau^{\max}_i\}$ on the maximum recovery time of any action, an optimal extreme point solution to \eqref{lp:LP} can be computed in polynomial time:

\begin{restatable}{fact}{restateLPsupport}
There exists an optimal solution to \eqref{lp:LP} which is supported in $(x_{i,\tau})_{i \in \A, \tau \in [\tau^{\max}_i]}$. Further, given an upper bound $\tau^{\max}$ on the maximum recovery time of any action, an optimal extreme point solution to \eqref{lp:LP} can be computed in polynomial time.
\end{restatable}

Let us introduce the notion of {\em supported} actions:
\begin{definition}[Supported actions] 
For any (potentially infeasible) solution $\x$ to \eqref{lp:LP}, we say that an action $i \in \A$ is {\em supported} in $\x$, denoted by $i \in \supp(\x)$, if there exists $\tau \in \mathbb{N}$ such that $x_{i,\tau} > 0$. 
\end{definition}

We introduce the following class of solutions:

\begin{definition}[Delay-feasible solutions] \label{def:delayfeasible}
We say that a vector $\x = (x_{i,\tau})_{i \in \A, \tau \in \mathbb{N}}$ is {\em delay-feasible}, if for each supported action $i \in \A(\x)$ there exists at most one non-zero variable $x_{i,\tau_i}$, and for this variable it holds $x_{i,\tau_i} = \nicefrac{1}{\tau_i}$. 

Further, we say that $\x$ is {\em almost-delay-feasible}, if it is delay-feasible, possibly with the exception of a single arm $\iota \in \A(\x)$, for which there exist at most two non-zero variables $x_{\iota,\tau_{\iota,a}}, x_{\iota,\tau_{\iota, b}}$, such that $x_{\iota,\tau_{\iota,a}} < \nicefrac{1}{\tau_{\iota,a}}$ and $x_{\iota,\tau_{\iota,b}} < \nicefrac{1}{\tau_{\iota,b}}$. In the case of almost-delay-feasible solutions, we refer to $\iota$ as the {\em irregular} arm.
\end{definition}

We remark that delay-feasible solutions subsume by definition almost-delay-feasible solutions and that a (almost-)delay-feasible solution is not necessarily a feasible solution to \eqref{lp:LP}.

As we prove in the next key-lemma, any extreme point solution of \eqref{lp:LP} is almost-delay-feasible:

\begin{restatable}[Sparsity pattern of extreme point solutions]{lemma}{restateAlmost}
\label{lem:almost}
Let $\x$ be any extreme point solution of \eqref{lp:LP}. Then, $\x$ is almost-delay-feasible. 
\end{restatable}
\begin{proof}
Let $\A(\x) \subseteq \A$ be the set of supported actions in an extreme point solution $\x$ of \eqref{lp:LP}, and let $|\A(\x)| = n'$. Let $\lp'$ be the LP that results by dropping from \eqref{lp:LP} all the variables and constraints that correspond to the non-supported actions in $\A \setminus \A(\x)$. Then, the solution $\x' = (x_{i,\tau})_{i \in \A(\x), \tau \in \mathbb{N}}$ (i.e., the solution $\x$ restricted to the arms of $\A(\x)$) is an extreme point solution of this reduced LP. In order to see that, let us assume that $\x'$ is not an extreme point and, hence, $\x' = \lambda \x'_1 + (1 - \lambda) \x'_2$ for some $\lambda \in (0,1)$ and $\x'_1,\x'_2$ feasible solutions to $\lp'$. Then, by padding $\x'_1$ and $\x'_2$ with zero variables for any $i \in \A \setminus \A(\x)$ and $\tau \in \mathbb{N}$, the resulting vectors, let $\x_1$ and $\x_2$, respectively, are feasible solutions to \eqref{lp:LP}. However, it also holds that $\x = \lambda \x_1 + (1 - \lambda) \x_2$, contradicting the fact that $\x$ is an extreme point of \eqref{lp:LP}. 

Now let us focus on the reduced formulation $\lp'$. By construction, since the actions in $\A(\x')$ are all supported in $\x'$, it has to be that $\|\x'\|_0 \geq n'$. Let $u$ be the number of variables of $\lp'$. By counting constraint \eqref{lp:total}, constraints \eqref{lp:arm} (one for each $i \in \A(\x')$), and $u$ non-negativity constraints, we can see that the number of constraints in $\lp'$ is $u + n' + 1$. Hence, since $\x'$ is an extreme point solution of $\lp'$, it must contain at most $n'+1$ non-zero variables (since at least $u-n'-1$ of the inequalities which are tight in $\x'$ must correspond to non-negativity constraints). This gives $\|\x'\|_0 \leq n'+1$ and, thus, we can conclude that $\|\x'\|_0 \in \{n', n'+1\}$. 

We now distinguish between two cases: 

{(a)} If $\|\x'\|_0 = n'$, then since every arm in the solution $\x'$ is supported and $\A(\x') = n'$, by pigeonhole principle each arm $i \in \A(\x')$ has to be supported by a unique non-zero variable $x'_{i,\tau_i} > 0$ for some $\tau_i \in \mathbb{N}$. Further, since exactly $n'$ out of the $u$ non-negativity constraints are not tight, then $n'$ out of the $n'+1$ inequalities in \eqref{lp:total} and \eqref{lp:arm} have to be met with equality. (i) In the case where these $n'$ constraints come from \eqref{lp:arm}, then for each $i \in \A(\x')$, there exists a unique delay $\tau_i$ and variable $x'_{i,\tau_i}$, such that $x'_{i,\tau_i} = \nicefrac{1}{\tau_i}$, implying that $\x'$ and, thus, $\x$ is delay-feasible. 
(ii) In the case where only $n'-1$ from the tight constraints in \eqref{lp:total} and \eqref{lp:arm} come from \eqref{lp:arm}, then for each corresponding arm $i \in \A(\x')$, there exists a unique delay $\tau_i$ and variable $x'_{i,\tau_i}$, such that $x'_{i,\tau_i} = \nicefrac{1}{\tau_i}$. Therefore, the arm $\iota \in \A(\x')$ which is associated with the non-tight constraint in \eqref{lp:arm} has to be an arm which is supported by a single variable $0 < x_{\iota,\tau_{\iota}} < \nicefrac{1}{\tau_{\iota}}$. In this case, $\x'$ and, thus, $\x$ is almost-delay-feasible, and $\iota$ corresponds to the irregular arm.

{(b)} In the case where $\|\x'\|_0 = n' + 1$, by a similar argument as above, each arm $i \in \A(\x')$ has to be supported by a unique non-zero variable $x'_{i,\tau_i} > 0$ for some $\tau_i \in \mathbb{N}$, with the exception of a single arm $\iota \in \A(\x')$, which is supported by two non-zero variables $x_{\iota,\tau_{\iota,a}}, x_{\iota,\tau_{\iota,b}} > 0$. Further, since $\|\x'\|_0 = n' + 1$, it has to be that all the constraints in \eqref{lp:total} and \eqref{lp:arm} are met with equality. Therefore, for any arm $i \in \A(\x') \setminus \{\iota\}$, there exists a unique delay $\tau_i$ and variable $x'_{i,\tau_i}$, such that $x'_{i,\tau_i} = \nicefrac{1}{\tau_i}$. Again, the solution $\x$ of \eqref{lp:LP} in this case is almost-delay-feasible, and $\iota$ corresponds to the irregular arm. 

Since delay-feasible solutions subsume almost-delay-feasible by definition, it follows that, in any case, $\x$ is almost-delay-feasible.
\end{proof}

\section{Improved Approximation Guarantees for Planning} \label{sec:algorithm}
For any instance of $k$-\RS, our algorithm (see \Cref{algo} below), called ``Randomize-Then-Interleave'', constructs a feasible planning schedule without requiring knowledge of the time horizon. We remark that for the planning setting of the problem, we treat the payoff of any arm under any possible delay as deterministic. 

The algorithm starts from computing an optimal extreme point solution $\x^*$ to \eqref{lp:LP}, and then uses this solution to determine a {\em critical} delay $\tau^*_i$ for each supported arm $i \in \supp(\x^*)$. By \Cref{lem:almost}, we know that $\x^*$ is almost-delay-feasible, which means that -- possibly with the exception of a single {irregular} arm $\iota$ -- for every arm $i \in \supp(\x^*) \setminus \{\iota\}$, there exists a unique $\tau_i$, such that $x^*_{i,\tau_i} = \nicefrac{1}{\tau_i} > 0$. Thus, \Cref{algo} sets this unique $\tau_i$ to be the critical delay of each arm $i \in \supp(\x^*) \setminus \{\iota\}$. In the case where an irregular arm $\iota \in \supp(\x^*)$ exists, then, by \Cref{def:delayfeasible}, there exist at most two distinct $\tau_{\iota,a}, \tau_{\iota, b}$ with $x^*_{\iota,\tau_{\iota,a}} > 0$ and $x^*_{\iota,\tau_{\iota,b}} \geq 0$, in which case the critical delay $\tau^*_{\iota}$ of $\iota$ is set to $\tau_{\iota,a}$ or $\tau_{\iota,b}$, with marginal probability $\tau_{\iota,a} x^*_{\iota,\tau_{\iota,a}}$ and $\tau_{\iota,b} x^*_{\iota,\tau_{\iota,b}}$, respectively. Notice that by constraints \eqref{lp:arm} of \eqref{lp:LP}, the above sampling process is well-defined, since $\tau_{\iota,a} x^*_{\iota,\tau_{\iota,a}} + \tau_{\iota,b} x^*_{\iota,\tau_{\iota,b}} \leq 1$. In the case where no $\tau^*_{\iota}$ is sampled, the irregular arm is removed from $\supp(\x^*)$.  After computing the critical delays, for each arm $i \in \supp(\x^*)$ the algorithm draws a random {\em offset} $r_i$ independently and uniformly from $\{0,1,\ldots, \tau^*_i -1 \}$.

\begin{algorithm2e} \label{algo}
\DontPrintSemicolon
\caption{Randomize-Then-Interleave}
{\color{blue}\tcc{Initialization phase}}
    Compute an optimal extreme point solution $\x^*$ to \eqref{lp:LP}.\;
    Let $\iota \in \A$ be the irregular arm (if such an arm exists).\;
    \For{{each arm} $i \in \supp(\x^*)\setminus \{\iota\}$}{
        Let $\tau^*_i$ be the unique $\tau \in \mathbb{N}$ which satisfies $x^*_{i, \tau} = \nicefrac{1}{\tau}$.\;}
    \If{an irregular arm $\iota \in \supp(\x^*)$ exists}
    {  
            Let $\tau_{\iota, a}, \tau_{\iota, b} \in \mathbb{N}$ be such that $x^*_{\iota, \tau_{\iota, a}} > 0$, $x^*_{\iota, \tau_{\iota, b}} \geq 0$, and $\tau_{\iota, a} \neq \tau_{\iota, b}$. \;
            Set $p_a \gets \tau_{\iota,a}  x^*_{\iota,\tau_{\iota,a}}$ and $p_b \gets \tau_{\iota,b}  x^*_{\iota,\tau_{\iota,b}}$. \;
            Sample $\tau^*_{\iota}$ from $\{\tau_{\iota,a},\tau_{\iota,b},\infty\}$ with marginals $p_a$, $p_b$, and $1 - p_a - p_b$, respectively. \;
            \If{$\tau^*_{\iota} = \infty$}{
            Remove $\iota$ from $\A(\x^*)$. \;
            }
    }
    \For{{each arm} $i \in \supp(\x^*)$}{
    Sample an {\em offset} $r_i$ uniformly at random from $\{0,1,\ldots, \tau^*_i -1 \}$. \;
    }
    {\color{blue}\tcc{Online phase}}
    \For{$t = 1,2, \ldots$}{
    Let $C_t \subseteq \A$ be the subset of {\em candidate} arms, defined as $C_t = \left\{i \in \supp(\x^*) \mid t~\textbf{mod}~\tau^*_i \equiv r_i\right\}$.\;
    Play the maximum-payoff subset of $k$ arms in $C_t$, namely, $A_t = \underset{S \subseteq C_t, |S| \leq k}{\text{argmax}} \sum_{i \in S} p_i(\sigma_{i,t}),$ where $\sigma_{i,t}$ ({\em actual delay}) is the time passed since arm $i$ was last played before $t$. \;
}
\end{algorithm2e}

After initialization, \Cref{algo} proceeds to its online phase where, at each round $t$, it first computes a subset of {\em candidate} arms $C_t$, namely, all arms $i \in \supp(\x^*)$ which satisfy $t~\textbf{mod}~\tau^*_i \equiv r_i$. Then, the algorithm computes the maximum-payoff subset $A_t$ of at most $k$ candidate arms, computed using their actual delays (the number of rounds passed since last play), and plays these arms.

\subsection{Approximation Analysis: Assuming Delay-Feasible LP Solutions} \label{sec:easy}
We now present an analysis of the approximation guarantee of \Cref{algo} for general $k$. In order to facilitate the presentation, we first focus on a special class of ``easy'' instances, where the optimal extreme point solution $\x^*$ of \eqref{lp:LP} is delay-feasible (i.e., there is no irregular arm). Consequently, we show how this assumption can be relaxed at the expense of additional technical effort. 

Recall that we denote by $\supp(\x^*)$ the set of arms that are supported in the solution $\x^*$. Notice that, since $\x^*$ is delay-feasible by assumption, the only source of randomness in \Cref{algo} is due to the random offsets. 
Since each offset $r_i$ is sampled independently and uniformly from $\{0,1, \ldots, \tau^*_{i} - 1\}$ and $i \in C_t$ if and only if $t~\textbf{mod}~\tau^*_i \equiv r_i$, the following fact is immediate:

\begin{fact} \label{fact:probCandidate}
Let $\{\tau^*_i\}_{i \in \supp(\x^*)}$ be the critical delays of the supported arms. At any round $t$, each arm $i \in \A(\x^*)$ belongs to $C_t$ independently with probability $\nicefrac{1}{\tau^*_i}$.
\end{fact}

Let us focus on the expected payoff collected by \Cref{algo} at any round $t \geq \tau^{\max}$. By our simplifying assumption that $\x^*$ is delay-feasible, for each supported arm $i \in \A(\x^*)$, there is a unique critical delay $\tau^*_i$, such that $x^*_{i,\tau^*_i} = \nicefrac{1}{\tau^*_i}$ is the only non-zero variable of $\x^*$ corresponding to $i$.

We construct a vector $\y = [0,1]^{n}$, such that $y_i = \nicefrac{1}{\tau^*_{i}}$ for each $i \in \A(\x^*)$, and $y_i = 0$, otherwise. Further, let us define a vector $\w \in [0, \infty)^n$ such that $w_i = p_i(\tau^*_{i})$ for each $i \in \A(\x^*)$, and $w_i = 0$, otherwise. Recall that for any fixed round $t \geq \tau^{\max}$, by definition of $C_t$ and the construction of the offsets, the events $\{i \in C_t\}_{i \in \A}$ are independent. Hence, we can think of $C_t$ as a random set, where each arm $i \in \A$ belongs to $C_t$ independently with probability $y_i$, namely, $C_t \sim \I(\y)$. 

Let $\underline{A}_t = \underset{S \subseteq C_t, |S| \leq k}{\text{argmax}} \sum_{i \in S} p_i(\tau^*_i)$ be a set of at most $k$ arms in $C_t$ which maximizes the total payoff measured with respect to the critical delays. By construction of $C_t$, the actual delay $\sigma_{i,t}$ of each $i \in \A(\x^*)$ is (deterministically) lower bounded by $\tau^*_i$, since any arm must be a candidate in order to be played (which happens exactly every $\tau^*_i$ rounds). Thus, for the expected collected payoff of any round $t \geq \tau^{\max}$ we have
\begin{align}
    \Ex{C_t \sim \I(\y)}{\max_{S \subseteq C_t, |S| \leq k}\sum_{i \in S} p_i(\sigma_{i,t})} \geq \Ex{C_t \sim \I(\y)}{\max_{S \subseteq C_t, |S| \leq k}\sum_{i \in S} p_i(\tau^*_i)} = \Ex{C_t \sim \I(\y)}{\w(\underline{A}_t)}. \label{eq:analysisActual}
\end{align}

Using \Cref{eq:uniform} in combination with \Cref{def:multilinear,def:concaveclosure}, we get
\begin{align}
    \Ex{C_t \sim \I(\y)}{\w(\underline{A}_t)} = \Ex{C_t \sim \I(\y)}{f_{\w,k}(C_t)} = F_{\w,k}(\y) \geq {\left(1 - \frac{k^k}{e^k k!} \right)} \cdot f^+_{\w,k}(\y), \label{eq:analysisOne}
\end{align}
where the last inequality follows by \Cref{lem:correlationgap}.

The following technical lemma relates $f^+_{\w,k}(\y)$ with optimal value of \eqref{lp:LP}:

\begin{restatable}{lemma}{restateConcaveupperbound} \label{lem:concaveupperbound}
For any vectors $\w \in [0,\infty)^n$ and $\y \in [0,1]^n$, let $f^+_{\w,k}(\y)$ be the concave closure of the weighted rank function of the rank-$k$ uniform matroid with weights $\w$, evaluated at $\y$. Then, if $\y$ satisfies $\| \y \|_1 \leq k$, it is the case that 
$f^+_{\w,k}(\y) \geq \sum_{i \in [n]} w_i y_i$.
\end{restatable}
\begin{proof}
Let us consider the intersection of the $n$-dimensional hypercube $[0,1]^n$ with the halfspace $H = \{\z \in \mathbb{R}^n \mid \sum_{i \in [n]} z_i \leq k\}$. It is easy to verify that every vertex $\bf v$ of the resulting polytope lies in $\{0,1\}^n$ and satisfies $\|{\bf v}\|_0 \leq k$. 
Then, since $\y \in [0,1]^n \cap H$, by standard arguments in polyhedral combinatorics, $\y$ can be expressed as a convex combination of characteristic vectors of vertices of $[0,1]^n \cap H$, each corresponding to a subset of $[n]$ of cardinality at most $k$. This convex combination induces a probability distribution $\D'(\y)$ over subsets of cardinality at most $k$ with marginals $\y$. Therefore, by definition of $f^+_{\w,k}$, we have
\begin{align*}
    f^+_{\w,k}(\y) = \sup_{\D(\y)} \Ex{S \sim \D(\y)}{f_{\w,k}(S)} \geq \Ex{S \sim \D'(\y)}{f_{\w,k}(S)} = \Ex{S \sim \D'(\y)}{\w(S)},
\end{align*}
where the last equality follows by the fact that $\D'(\y)$ is only supported by sets of cardinality at most $k$. 

Finally, using the fact that $\D'(\y)$ respects the marginals $\y$, we can conclude
\begin{align*}
    f^+_{\w,k}(\y) \geq \Ex{S \sim \D'(\y)}{\w(S)} = \sum_{i \in [n]} w_i \Prob{S \sim \D'(\y)}{i \in S} = \sum_{i \in [n]} w_i  y_i.
\end{align*}
\end{proof}

Let $V^*$ be the optimal value of \eqref{lp:LP}. Since $\| \y \|_1 \leq k$, then by applying \Cref{lem:concaveupperbound} with vectors $\w$ and $\y$ as constructed above, we have that  
\begin{align} 
    f^+_{\w,k}(\y) \geq \sum_{i \in [n]} w_i \cdot y_i = \sum_{i \in \A(\x^*)} p_i(\tau^*_{i}) \cdot x^*_{i, \tau^*_i} = V^*.  \label{eq:analysisTwo}
\end{align}

By combining \Cref{eq:analysisActual,eq:analysisOne,eq:analysisTwo} with \Cref{lem:lpupperbound}, it follows that, for any round $t \geq \tau^{\max}$, the expected payoff collected by \Cref{algo} is at least ${\left(1 - \nicefrac{k^k}{e^k k!} \right)}$-times the average optimal payoff. By using linearity of expectations together with the fact that all payoffs are bounded in $[0,1]$, we can prove a long-run ${\left(1 - \nicefrac{k^k}{e^k k!} \right)}$-approximation guarantee (in expectation) for instances where $\x^*$ is delay-feasible. 

In the rest of this section, we show how the above analysis can be extended for the general case where $\x^*$ is almost-delay-feasible (i.e., an irregular arm exists).


\subsection{Approximation Analysis: Handling the Irregular Arm} \label{sec:irregular}

We now extend the analysis of \Cref{sec:easy} and relax the assumption that the optimal extreme point solution of \eqref{lp:LP} is delay-feasible. 

Let $\x^*$ be an optimal solution to \eqref{lp:LP} and let $\iota \in \A(\x^*)$ be the irregular arm with $\tau^*_{\iota,a}$ and $\tau^*_{\iota,b}$ being the two potential critical delays such that $x^*_{\iota,\tau_{\iota,a}}, x^*_{\iota,\tau_{\iota,b}} \geq 0$ (with at least one of $x^*_{\iota,\tau_{\iota,a}}$ and $ x^*_{\iota,\tau_{\iota,b}}$ being strictly positive, by \Cref{def:delayfeasible}). Instead of conditioning on the sampled critical delay of the irregular arm $\iota$, we construct two (fictitious) parallel copies $\iota_a$ and $\iota_b$ -- one for each possible realization (and associated payoff). Let us define $\A' = \supp(\x^*) \setminus \{\iota\} \cup \{\iota_a, \iota_b\}$ to be the set of supported arms of $\x^*$, where we replace the irregular arm $\iota$ with the two parallel copies $\iota_a$ and $\iota_b$. 

In the next lemma, we characterize the marginal probability that any arm in $\A'$ is added to the set of candidate arms at any round. Notice that, as opposed to \Cref{fact:probCandidate}, the randomness in this case is taken over both the sampling of a critical delay for the irregular arm as well as the random offsets. 

\begin{restatable}{lemma}{restateProbability} \label{lem:probability}
For any arm $i\in \A'$ and at any round $t$, we have that  $\Pro{i \in C_t \text{ and } \tau^*_i = \tau} = x^*_{i,\tau}$.
\end{restatable}

Notice that, for the parallel copies of the irregular arm, the events $\{\iota_a \in C_t\}$ and $\{\iota_b \in C_t\}$ are now dependent and, in particular, {\em mutually exclusive} (since, eventually, at most one of them is realized). However, we are able to show that the expected payoff collected at any round $t \geq \tau^{\max}$ can only decrease if the two events were instead independent with the same marginal probabilities. 

Let us relabel the arms in $\A'$ in non-increasing order of $p_i(\tau^*_i)$ (breaking ties arbitrarily). Note that by exchanging $\iota$ with the two copies $\iota_a$ and $\iota_b$ the cardinality of $\A'$ can be at most $n+1$. 

Let us construct a $(n+1)$-dimensional vector $\y \in [0,1]^{n+1}$, such that $y_i = x^*_{i, {\tau^*_i}}$ for any $i \in \A'$, and $y_i = 0$, otherwise. Further, we construct a weight vector $\w \in [0,\infty)^{n+1}$, such that $w_i = p_i({\tau^*_i})$ for any $i \in \A'$, and $w_i = 0$, otherwise. Let $\mathcal{D}(\y)$ be a distribution over subsets of $\A'$, where every arm $i \in \A' \setminus \{\iota_a, \iota_b\}$ (except for $\iota_a$ and $\iota_b$) is added to a set $S \sim \mathcal{D}(\y)$, independently, with probability $x^*_{i,\tau^*_{i}}$. The arms $\iota_a$ and $\iota_b$ are also added to $S \sim \mathcal{D}(\y)$ with marginal probabilities $x^*_{\iota_a, \tau^*_{\iota_a}}$ and $x^*_{\iota_b, \tau^*_{\iota_b}}$, respectively, yet not independently (since they are mutually exclusive). 
Therefore, by the above definitions and \Cref{lem:probability}, the set of candidate arms at any fixed round $t \geq \tau^{\max}$ is distributed as $C_t \sim \mathcal{D}(\y)$. 

Let us fix any round $t \geq \tau^{\max}$. Similarly to \Cref{sec:easy}, by using the fact that the actual delay $\sigma_{i,\tau}$ of each arm $i \in \A'$ is deterministically lower bounded by its critical delay $\tau^*_i$, we can lower bound the expected payoff collected by our algorithm at time $t$ as 
\begin{align}
\Ex{C_t \sim \D(\y)}{\max_{S \subseteq C_t, |S| \leq k}\sum_{i \in S} p_i(\sigma_{i,t})} \geq \Ex{C_t \sim \D(\y)}{\max_{S \subseteq C_t, |S| \leq k}\sum_{i \in S} p_i(\tau^*_i)} = \Ex{C_t \sim \D(\y)}{f_{\w,k}(C_t)}, \label{eq:irreg1}
\end{align}
where $\D(\y)$ is defined as described above.

As we show in the following lemma, the RHS of the above inequality can only decrease, if we replace $C_t \sim \D(\y)$ with $C_t \sim \I(\y)$, namely, if we assume that the two copies $\iota_a$ and $\iota_b$ are added to $C_t$ independently, but with the same marginals ($y_{\iota_a}$ and $y_{\iota_b}$, respectively). As we show in the following lemma, the above property in fact holds for any submodular function (including $f_{\w,k}$):


\begin{restatable}{lemma}{restateMutually} \label{lem:mutually}
Let $f: 2^{[n]} \rightarrow \mathbb{R}$ be a submodular function over a ground set of $n$ elements, where $n \geq 2$. Let $\I(\y)$ be a distribution over $2^{[n]}$, where each element $i$ is added in set $S \sim \I(\y)$, independently, with probability $y_i$. Further, let $\mathcal{D}(\y)$ be a distribution over $2^{[n]}$, where each element $i \in [n] \setminus \{\iota_a, \iota_b\}$ is added to a set $S \sim \mathcal{D}(\y)$, independently, with probability $y_i$, except for two distinct elements $\iota_a, \iota_b$, whose addition to $S$ is mutually exclusive but with the same marginals $y_{\iota_a}$ and $y_{\iota_b}$, respectively. In this case, we have  
\begin{align*}
    \Ex{S \sim \D(\y)}{f(S)} \geq \Ex{S \sim \I(\y)}{f(S)}.
\end{align*}
\end{restatable}
\begin{proof}
By definition of the distribution $\mathcal{D}(\y)$, since the events $\{\iota_a \in S\}$ and $\{\iota_b \in S\}$ are mutually exclusive, we have that 
$\Prob{S \sim \mathcal{D}(\y)}{\{\iota_a \in S\} \cap \{\iota_b \notin S\}} = \Prob{S \sim \mathcal{D}(\y)}{\iota_a \in S} = y_{\iota_a}$ and $\Prob{S \sim \mathcal{D}(\y)}{\{\iota_a \notin S\} \cap \{\iota_b \in S\}} = \Prob{S \sim \mathcal{D}(\y)}{\iota_b \in S} = y_{\iota_b}$. Further, we have that $\Prob{S \sim \mathcal{D}(\y)}{\{\iota_a \notin S\} \cap \{\iota_b \notin S\}} = 1 - \Prob{S \sim \mathcal{D}(\y)}{\{\iota_a \in S\} \cup \{\iota_b \in S\}} = 1 - y_{\iota_a} - y_{\iota_b}$.

Using the above, we can rewrite the LHS of the desired inequality as
\begin{align*}
    \Ex{S \sim \D(\y)}{f(S)} = \sum_{S' \subseteq [n] \setminus \{\iota_a, \iota_b\}} \prod_{i \in S'} y_i \prod_{\substack{i \notin S' \\ i \neq \iota_a, \iota_b}}(1 - y_i)\Big((1- y_{\iota_a} - y_{\iota_b})f(S') + y_{\iota_a} f(S' + \iota_a) + y_{\iota_b} f(S' + \iota_b)\Big).
\end{align*}

Now, for any fixed set $S' \subseteq [n] \setminus \{\iota_a, \iota_b\}$, the following identities can be easily verified: 
\begin{align*}
(1- y_{\iota_a} - y_{\iota_b})f(S') &= (1-y_{\iota_a})(1-y_{\iota_b}) f(S') - y_{\iota_a}y_{\iota_b} f(S') \\
y_{\iota_a} f(S' + \iota_a) &= y_{\iota_a}(1 - y_{\iota_b}) f(S' + \iota_a) + y_{\iota_a}y_{\iota_b}f(S' + \iota_a)\\
y_{\iota_b} f(S' + \iota_b) &= y_{\iota_b}(1 - y_{\iota_a}) f(S' + \iota_b) + y_{\iota_a}y_{\iota_b}f(S' + \iota_b).
\end{align*}
By combining the above identities, we get that 
\begin{align*}
&(1- y_{\iota_a} - y_{\iota_b})f(S') + y_{\iota_a} f(S' + \iota_a) + y_{\iota_b} f(S' + \iota_b) \\
&= (1-y_{\iota_a})(1-y_{\iota_b}) f(S') + y_{\iota_a}(1 - y_{\iota_b}) f(S' + \iota_a) + y_{\iota_b}(1 - y_{\iota_a}) f(S' + \iota_b) + y_{\iota_a}y_{\iota_b} \left( f(S' + \iota_a) + f(S' + \iota_b) - f(S')\right) \\ 
&\geq (1-y_{\iota_a})(1-y_{\iota_b}) f(S') + y_{\iota_a}(1 - y_{\iota_b}) f(S' + \iota_a) + y_{\iota_b}(1 - y_{\iota_a}) f(S' + \iota_b) + y_{\iota_a}y_{\iota_b} f(S' + \iota_a + \iota_b),
\end{align*}
where, in the inequality, we use that $f(S') + f(S' + \iota_a + \iota_b) \leq f(S'+\iota_a) + f(S'+\iota_b)$, by submodularity of $f$. 

By applying the above inequality for each $S' \subseteq [n] \setminus \{\iota_a, \iota_b\}$, we can conclude that
\begin{align*}
    \Ex{S \sim \D(\y)}{f(S)} &= \sum_{S' \subseteq [n] \setminus \{\iota_a, \iota_b\}} \prod_{i \in S'} y_i \prod_{\substack{i \notin S' \\ i \neq \iota_a, \iota_b}}(1 - y_i)\Big((1- y_{\iota_a} - y_{\iota_b})f(S') + y_{\iota_a} f(S' + \iota_a) + y_{\iota_b} f(S' + \iota_b)\Big) \\
    &\geq \sum_{S' \subseteq [n] \setminus \{\iota_a, \iota_b\}} \prod_{i \in S'} y_i \prod_{\substack{i \notin S' \\ i \neq \iota_a, \iota_b}}(1 - y_i)\Big((1-y_{\iota_a})(1-y_{\iota_b}) f(S') + y_{\iota_a}(1 - y_{\iota_b}) f(S' + \iota_a) \\ 
    &\qquad\qquad\qquad\qquad\qquad\qquad\qquad\qquad + y_{\iota_b}(1 - y_{\iota_a}) f(S' + \iota_b) + y_{\iota_a}y_{\iota_b} f(S' + \iota_a + \iota_b)\Big) \\
    &= \sum_{S \subseteq [n]} \prod_{i \in S} y_i \prod_{i \notin S}(1 - y_i) f(S)\\
    &= \Ex{S \sim \I(\y)}{f(S)},
\end{align*}
thus proving the lemma. 
\end{proof}

By combining inequality \eqref{eq:irreg1} with \Cref{lem:mutually} for $n+1$ elements (recall, the function $f_{\w,k}$ is submodular), the expected reward collected at any round $t \geq \tau^{\max}$ can be lower bounded as 
$$
\Ex{C_t \sim \D(\y)}{\max_{S \subseteq C_t, |S| \leq k}\sum_{i \in S} p_i(\sigma_{i,t})} \geq \Ex{C_t \sim \D(\y)}{f_{\w,k}(C_t)} \geq \Ex{C_t \sim \I(\y)}{f_{\w,k}(C_t)} = F_{\w,k}(\y).
$$
In addition, by combining \Cref{lem:correlationgap} with \Cref{lem:concaveupperbound} for the above choice of vectors $\w$ and $\y$  (again, it holds $\| \y \|_1 \leq k$), for any $t \geq \tau^{\max}$, it follows that 
$$
\Ex{C_t \sim \D(\y)}{\max_{S \subseteq C_t, |S| \leq k}\sum_{i \in S} p_i(\sigma_{i,t})} \geq F_{\w,k}(\y) \geq {\left(1 - \frac{k^k}{e^k k!} \right)} \cdot f^+_{\w,k}(\y) \geq {\left(1 - \frac{k^k}{e^k k!} \right)} V^*, 
$$
where $V^*$ is the optimal value of \eqref{lp:LP}.

Finally, using the fact that $V^* \geq \nicefrac{\opt(T)}{T}$ by \Cref{lem:lpupperbound} and applying linearity of expectations over all rounds $t$, the following result follows directly:

\begin{restatable}{theorem}{restateMainone} \label{thm:main1}
For any instance of $k$-\RS, where $\opt(T)$ is the optimal expected payoff that can be collected in $T$ rounds, the total expected payoff collected by \Cref{algo} is at least 
$$
{\left(1 - \frac{k^k}{e^k k!} \right)} \opt(T) - \mathcal{O}(k \cdot \tau^{\max}) ~~~\approx~~~ \left(1 - \frac{1}{\sqrt{2\pi k}}\right) \opt(T) - \mathcal{O}(k \cdot \tau^{\max}).
$$
\end{restatable}

We remark that, although the asymptotic approximation guarantee of \Cref{algo} is the same as the one in \cite{SLZZ21}, the constants of our algorithm are significantly better. In particular, for $k \in \{1,2,3,4,5,10\}$, the guarantees of \cite{SLZZ21} and \Cref{algo} are $\{0.25, 0.33,0.40,0.44,0.46,0.57\}$ and $\{0.63, 0.72,0.77,0.80,0.82,0.87\}$, respectively.

\section{Online Learning with Sublinear Regret}
\label{sec:learning}
In this section, we consider the learning setting of $k$-\RS where, for every arm $i \in \A$ and delay $\tau$, the payoff for playing arm $i$ under delay $\tau$ is drawn independently from a distribution of mean $\f_i(\tau)$, bounded in $[0,1]$. Assuming semi-bandit feedback on the realized payoffs, our goal is to show that there exists a bandit variant of \Cref{algo} with sublinear $\gamma_k$-approximate regret guarantee for $k$-\RS, defined as 
$$
\text{Reg}_{\gamma_k}(T) = \gamma_k \opt(T) - \text{R}(T),
$$
where $\gamma_k = 1 - \nicefrac{k^k}{e^k k!}$ is the (multiplicative) approximation guarantee of \Cref{algo} and $R(T)$ is the total expected payoff collected by the bandit algorithm in $T$ rounds. 

\paragraph{Robustness of \Cref{algo}.} As a first step, we study the robustness of \Cref{algo} under small perturbations of the payoff functions; in particular, we analyze its approximation guarantee when it runs using estimates $\{\widehat{p}_i\}_{i \in \A}$ instead of the actual payoffs $\{{p}_i\}_{i \in \A}$. We remark that the monotonicity of the (estimated) payoff functions is {\em not required} for the correctness of \Cref{algo}, but only for proving its approximation guarantee.

\begin{restatable}{lemma}{restateRobustness} \label{lem:robustness}
Let $\widehat{p}_{i}$ be an estimate of the payoff function $p_{i}$ for each $i \in \A$, such that $|\widehat{p}_{i}(\tau) - p_{i}(\tau)| \leq \epsilon, \forall \tau \in \mathbb{N}$ for some $\epsilon \in (0,1)$. The expected payoff collected in $T$ rounds by \Cref{algo}, when it uses the estimates instead of the actual payoff functions, is lower bounded by 
$$\gamma_k \opt(T) - \mathcal{O}(k \cdot \epsilon \cdot T + k \cdot \tau^{\max}).$$ 
\end{restatable}

\paragraph{Estimation of the payoff functions.}
By using standard concentration results \cite{Hoeffding}, we bound the number of samples required for each arm-delay pair in order to get an $\epsilon$-estimate $\widehat{p}_i$ for the mean payoff function ${p}_i$ of every arm $i \in \A$. 

\begin{restatable}{lemma}{restateEstimation} \label{lem:estimation}
For any $\epsilon,\delta \in (0,1)$, let $\widehat{p}_i(\tau)$ be the empirical average of samples drawn from $p_i(\tau)$ for every arm $i \in \A$ and delay $\tau$, using $m = \frac{1}{2\epsilon^2} \ln\left(\frac{2 \tau^{\max} n}{\delta}\right)$ samples, where $n$ is the number of arms. Then, with probability $1-\delta$, it holds $|\widehat{p}_{i}(\tau) - p_{i}(\tau)| \leq \epsilon$ for every $i \in \A$ and $\tau \in [\tau^{\max}]$.
\end{restatable}

The following observation provides an upper bound on the number of rounds required in order to get $m$ samples for each arm-delay pair in any $k$-\RS instance:

\begin{restatable}{fact}{restateFactEstimation} \label{fact:estimation}
In any $k$-\RS instance, the number of rounds required to collect $m$ independent samples from $p_i(\tau)$, for each arm $i \in \A$ and delay $\tau \in [\tau^{\max}]$, is upper-bounded by $\mathcal{O}(\nicefrac{n m (\tau^{\max})^2}{k})$.
\end{restatable}

\paragraph{Bandit algorithm with sublinear regret.} By combining the above elements, it is easy to design an Explore-then-Commit (ETC) variant of \cref{algo}, achieving sublinear regret in the bandit setting. For simplicity, we assume here that the time horizon $T$ is known to the player -- an assumption that can be relaxed using the doubling trick method \cite{LS18}. 

Let $m = \nicefrac{1}{2\epsilon^2} \ln\left(\nicefrac{2 \tau^{\max} n}{\delta}\right)$ for some $\epsilon, \delta \in (0,1)$. In the first $\mathcal{O}(\nicefrac{n m (\tau^{\max})^2}{k})$ rounds, the bandit algorithm collects $m$ independent samples from $p_i(\tau)$, for each $i \in \A$ and $\tau \in [\tau^{\max}]$. Then, by taking the empirical average of these samples, it constructs estimates $\widehat{p}_i(\tau)$ for each $i$ and $\tau$. For the rest of the time horizon, the algorithm simulates \Cref{algo}, using the empirical estimates $\widehat{p}_i(\tau)$ in place of the (initially unknown) payoff functions.

By setting the parameters $\epsilon, \delta$ accordingly and combining \Cref{lem:robustness,lem:estimation} with Fact \ref{fact:estimation}, we prove the following result:
\begin{restatable}{theorem}{restateRegret} \label{thm:regret}
There exists a learning adaptation of \Cref{algo} in the semi-bandit setting, for which the $\gamma_k$-approximate regret for $T$ rounds can be upper-bounded as
$$
\mathcal{O}\left(n^{1/3} \cdot (k \cdot \tau^{\max})^{2/3} \ln^{1/3}(\tau^{\max} n T) \cdot T^{2/3}  + k \cdot \tau^{\max}\right).
$$
\end{restatable}

\section*{Conclusion and Further Directions}

Following the recent line of work on non-stationary bandits with recharging payoffs, we revisited one of the most general formulations studied in the area and significantly improved the state-of-the-art for the planning setting. In particular, when at most one arm is played per round, we designed an algorithm that collects at least $63\%$ (asymptotically and in expectation) of the optimal payoff, improving the best-known $25\%$, due to prior work. By providing results on the robustness and sample complexity of our algorithm, we transformed the latter into a bandit algorithm with sublinear regret. 
Our work leaves a number of interesting open questions. An approximation guarantee of $\left(1 - \nicefrac{1}{e}\right)$ has been proved to be the best achievable -- under standard complexity assumptions -- for a number combinatorial problems. Therefore, an immediate future direction would be to either provide a matching hardness result for the planning setting of the problem, or to further improve on the best-known approximation. Another interesting question would be to explore whether one can design a bandit adaptation of our algorithm with improved regret guarantees. Finally, the empirical evaluation of our algorithm on real (or artificial) data is another interesting direction.

\bibliographystyle{alpha}
\bibliography{ref.bib}
\newpage

\appendix

\section{Related Work} \label{sec:relatedwork}

Since the work of Thompson \cite{T33} and later \cite{LR85}, the (stochastic) multi-armed bandits (MAB) problem in stationary environments has been thoroughly studied (see \cite{BB12,LS18,ACBF02} and references therein). In an attempt to capture non-stationary environments, where the payoff distributions change over time, a number of generalizations of MAB -- which are also special cases of reinforcement learning (RL) -- have been proposed. Two particularly interesting classes are {\em rested} and {\em restless} bandits \cite{TL12}, where each action is associated with a state-machine and has a different payoff distribution depending on its current state. In rested bandits \cite{Gittins79, TL12}, the state of each action (hence, its payoff distribution) changes stochastically only when the arm is played, while in restless bandits \cite{Whittle88, TL12} the state changes at every round, independently of the player's actions. Optimal algorithms for rested and restless bandits essentially result from solving Bellman’s equations \cite{Bertsekas} -- an approach which requires an exponentially large state space in the number of arms. In fact, Papadimitriou and Tsitsiklis \cite{PT99} show that it is PSPACE-hard to optimally solve (or even approximate) restless bandits, even in the simpler case where state transitions are deterministic. One of the first attempts to provide constant approximation algorithms for special cases of restless bandits is due to Guha et al. \cite{GMS10}, who give a $2$-approximation for a special case of restless bandits, which they call ``monotone'' bandits. Another interesting class of non-stationary bandits is that of {\em sleeping} bandits \cite{KNMS10,saha2020improved}, which generalizes the standard MAB setting, given that only an adversarially chosen subset of actions can be played at each round. Finally, we address the reader to \cite{Besbes,Levine,Keskin,auer2019adaptively} for additional non-stationary bandit models.

Shifting our focus to the case of ``recharging'' (or ``recovering'') payoffs, Immorlica and Kleinberg \cite{KI18} first consider the $1$-\RS problem, under the additional assumption that the payoff functions are weakly concave and weakly increasing over the whole time horizon. For this problem, the authors provide a PTAS that achieves a $(1-\epsilon)$-approximation (asymptotically) to the optimal planning, and show how it can be adapted into a learning algorithm with sublinear regret guarantees. To achieve their approximation, they first compute an ``optimal'' playing frequency for each arm via a concave relaxation, and then construct an approximate planning schedule, in which the rate of playing each arm closely approximates this frequency. In this direction, they authors combine enumeration techniques with the novel technique of {\em interleaved rounding}. The main idea behind the latter is that each arm is first associated with a particular (randomly perturbed) sequence of real numbers and, then, the arms are played sequentially according to the order they appear on the real line. 

Basu et al.\ \cite{BSSS19} study a variant of stochastic multi-armed bandits where, after each play, an arm becomes unavailable for a known number of rounds (blocking time). As they show, for the case of one arm per round, a simple greedy heuristic, which plays, at each round, the arm of highest expected payoff among the non-blocked arms, yields (asymptotically) a $\left(1 - \nicefrac{1}{e}\right)$-approximation. In order to prove this guarantee, the authors compare the payoff collected by their algorithm with the optimal solution of an LP-relaxation, by carefully characterizing the latter via analyzing its optimality conditions. Although not stated explicitly in \cite{BSSS19}, the ``hard'' constraint of not pulling a blocked arm is equivalent to choosing the payoff function of each arm to be a Heaviside step function parameterized by its blocking time and scaled by its baseline mean payoff. In other words, the payoff drops to zero for a fixed number of rounds (which, in this case coincides with the recovery time), and then rises up to a baseline. Since pulling a zero-payoff arm can only harm the overall solution, the two problems are (without loss of generality) equivalent, and \cite{BSSS19} becomes a special case of $1$-\RS. Notice, however, that the greedy approach of \cite{BSSS19} fails to provide any non-trivial approximation guarantee for the $1$-\RS setting.

Shortly after its introduction, the model of Basu et al.\ \cite{BSSS19} has been also studied in contextual \cite{BPCS20}, adversarial \cite{BCMT20}, and combinatorial environments \cite{APBCS21,PC21}. In \cite{PC21}, Papadigenopoulos and Caramanis study the blocking setting in the regime where more than one arms can be played at each round, subject to matroid constraints. For the planning setting, they provide a $(1 - \nicefrac{1}{e})$-approximation algorithm (asymptotically and in expectation) for any matroid. Their technique, called {\em interleaved scheduling} -- which is similar in spirit with the {\em interleaved rounding} method of \cite{KI18} -- is essentially a time-correlated arm-sampling scheme, which provides the following two guarantees: (i) at any fixed round, each arm $i$ is sampled with probability $\nicefrac{1}{\tau_i}$ (the maximum possible in any optimal solution), where $\tau_i$ is the blocking time of the arm. (ii) Every time an arm is sampled, its payoff must have returned to its baseline (i.e., is never blocked). Notice that the algorithm of \cite{PC21}, which essentially chooses a maximum-payoff independent set among the sampled arms of each round, also provides an alternative method for reaching the $(1 - \nicefrac{1}{e})$-approximation of \cite{BSSS19} for the special case of one arm per round. 

Cella and Cesa-Bianchi \cite{CCB19} consider a similar model to $1$-\RS, where they assume that all the actions have identical payoff functions up to different scaling (i.e., baseline payoff), but different and unknown recovery times. For the planning problem, by focusing on a simple class of {\em periodic ranking policies}, the authors develop an algorithm, accompanied by a worst-case approximation guarantee that depends on characteristics of the input (payoff functions, recovery times, and baselines). Finally, other special cases and variations of the $k$-\RS problem include payoff functions described by linear dynamics \cite{LKLM21}, sampled by Gaussian processes \cite{PBG19}, and more \cite{LCCBGB21}.

Simchi-Levi et al.\ \cite{SLZZ21} provide the first $\mathcal{O}(1)$-approximation the general case of the $k$-\RS problem. Specifically, by focusing on a particular class of policies, which they call {\em purely periodic}, the authors provide a long-run $\left(1 - \mathcal{O}(\nicefrac{1}{\sqrt{k}})\right)$-approximation for any instance of $k$-\RS. A key-element in both Immorlica and Kleinberg \cite{KI18} and, subsequently, Simchi-Levi et al.\ \cite{SLZZ21} is the construction of a concave relaxation of the optimal solution, based on the critical notion of {\em rate of return}: a piecewise linear function, which denotes the maximum payoff that can be achieved asymptotically by playing an arm at most a specific fraction of rounds. The authors in \cite{KI18} provide an FPTAS for computing an $(1-\epsilon)$-approximate solution to their relaxation (for the case where $k=1$), while, in \cite{SLZZ21}, such a solution is efficiently computed assuming finite recovery times.

\section{Structural Properties of a Natural LP Relaxation: Omitted Proofs}

\restateUpperBound*

\begin{proof}
Let $n_{i,\tau}$ be the number of times arm $i$ is played under delay $\tau \in \mathbb{N}$ in an optimal solution. By fixing $x_{i,\tau} = \frac{n_{i,\tau}}{T}$, we have that
\begin{align*}
    \opt(T) = \sum_{i \in \A} \sum_{\tau \in \mathbb{N}} p_i(\tau) \cdot n_{i,\tau} = T \cdot \sum_{i \in \A} \sum_{\tau \in \mathbb{N}} p_i(\tau) \cdot \frac{n_{i,\tau}}{T} = T \cdot \sum_{i \in \A} \sum_{\tau \in \mathbb{N}} p_i(\tau) \cdot x_{i, \tau}.
\end{align*}
Further, note that in any instance where at most $k$ arms can be played at each round, the total number of arm plays (independently of the delay) cannot be more than $k \cdot T$. Thus, we get
\begin{align*}
    \sum_{i \in \A} \sum_{\tau \in \mathbb{N}} x_{i,\tau} = \frac{1}{T} \cdot \sum_{i \in \A} \sum_{\tau \in \mathbb{N}} n_{i,\tau} \leq  \frac{1}{T} \cdot k \cdot T \leq k.
\end{align*}

Finally, consider any action $i \in \A$ which is played at time $t$ under delay $\tau$. The number of rounds that are occupied for the above specific play belong to the interval $I = \{t-\tau+1, t\}$, with $|I| = \tau$ (note that for $\tau = 1$ only round $t$ is occupied). Recalling our assumption that all actions start with delay $1$ at round $t=1$, then for any such interval we have that $I \subseteq [T]$. Combining the above with the fact that, for any action $i \in \A$, these intervals cannot overlap, we get that 
\begin{align*}
    \sum_{\tau \in \mathbb{N}} \tau \cdot n_{i,\tau} \leq T \Longrightarrow \sum_{\tau \in \mathbb{N}} \tau \cdot x_{i,\tau} \leq 1, \qquad \forall i \in \A. 
\end{align*}
Therefore, since $\x$ is a feasible solution to \eqref{lp:LP} with value equal to $\nicefrac{\opt(T)}{T}$, the proof follows directly. 
\end{proof}

\restateLPsupport*
\begin{proof}
Clearly, proving that any optimal solution to \eqref{lp:LP} is without loss of generality supported in $(x_{i,\tau})_{i \in \A, \tau \in [\tau^{\max}_i]}$ implies that we can restrict ourselves to a polynomial-size version of \eqref{lp:LP}, which is constructed by dropping all variables except for $(x_{i,\tau})_{i \in \A, \tau \in [\tau^{\max}_i]}$. In this way, an optimal extreme point solution to this restricted LP can be computed efficiently using a standard LP solver. 

Let $\x$ be any optimal solution to \eqref{lp:LP} such that there exists an arm $i$ and delay $\tau' > \tau^{\max}_i$ satisfying $x_{i, \tau'} > 0$. Note that by definition of the recovery time, it must hold $p_i(\tau') = p_i(\tau^{\max}_i)$. Thus, we can construct a solution $\x'$ which is identical to $\x$, except for the fact that $x'_{i,\tau'} = 0$ and $x'_{i,\tau^{\max}_i} = x_{i,\tau^{\max}_i} + x_{i,\tau'}$. Since $p_i(\tau') = p_i(\tau^{\max}_i)$ the solutions $\x'$ and $\x$ have exactly the same value, while $\x'$ trivially satisfies constraints \eqref{lp:total}. In addition, constraints \eqref{lp:arm} are still satisfied by $\x'$ for any arm different than $i$. For arm $i$, we have 
\begin{align*}
\sum_{\tau \in \mathbb{N}} \tau \cdot x'_{i, \tau} &= \sum_{\tau \in \mathbb{N} \setminus \{\tau', \tau^{\max}_i\}} \tau \cdot x_{i, \tau} + \tau^{\max}_i \cdot x'_{i, \tau^{\max}_i} \\
&= \sum_{\tau \in \mathbb{N} \setminus \{\tau', \tau^{\max}_i\}} \tau \cdot x_{i, \tau} + \tau^{\max}_i \cdot x_{i,\tau^{\max}_i} +  \tau^{\max}_i \cdot x_{i,\tau'} \\
&\leq \sum_{\tau \in \mathbb{N}} \tau \cdot x_{i,\tau} \\
&\leq 1,
\end{align*}
where the first inequality follows by the fact that $\tau' > \tau^{\max}_i$ and the second by feasibility of $\x$.

By repeating the above process for every arm $i$ and delay $\tau > \tau^{\max}_i$ with $x_{i, \tau'} > 0$, we can transform any solution to \eqref{lp:LP} to an equivalent solution of the desired form. 
\end{proof}


\section{Improved Approximation Guarantees for Planning: Omitted Proofs} \label{sec:analysis}

\restateProbability*
\begin{proof}
Fix any arm $i\in \A'$ and round $t$. By the sampling process of critical delays using the solution $\x^*$, we have that $\Pro{\tau^*_i = \tau} = \tau \cdot x^*_{i,\tau}$. Notice that, when arm $i$ is not the irregular arm, by \Cref{def:delayfeasible} this probability is identically one. Since the offset $r_i$ is sampled uniformly at random from $\{0,1, \ldots, \tau^*_{i} - 1\}$ and $i \in C_t$ if and only if $t~\textbf{mod}~\tau^*_i = r_i$, we get that $\Pro{i \in C_t \mid \tau^*_i = \tau} = \nicefrac{1}{\tau}$. Hence, we can conclude that 
\begin{align*}
    \Pro{i \in C_t \text{ and } \tau^*_i = \tau} = \Pro{\tau^*_i = \tau} \cdot \Pro{i \in C_t \mid \tau^*_i = \tau} = \left(\tau \cdot x^*_{i,\tau}\right) \cdot \frac{1}{\tau} = x^*_{i, \tau}. 
\end{align*}
\end{proof}


\section{Online Learning with Sublinear Regret: Omitted Proofs}

\restateRobustness*

\begin{proof}
Let $\x^*$ and $\widehat{\x}$ be the optimal solution of \eqref{lp:LP} in the case where the payoff functions in the objective are given by $\{p_i\}_{i \in \A}$ and $\{\widehat{p}_i\}_{i \in \A}$, respectively. Let also $V^*$ and $\widehat{V}$ be the optimal value of \eqref{lp:LP} in each of the aforementioned cases. We have 
\begin{align*}
\widehat{V} = \sum_{i \in \A} \sum_{\tau \in \mathbb{N}} \widehat{p}_{i}(\tau) \cdot \widehat{x}_{i,\tau} \geq \sum_{i \in \A} \sum_{\tau \in \mathbb{N}} \widehat{p}_{i}(\tau) \cdot x^*_{i,\tau} \geq \sum_{i \in \A} \sum_{\tau \in \mathbb{N}} \left( {p}_{i}(\tau) - \epsilon \right) \cdot x^*_{i,\tau} \geq V^* - \epsilon \cdot k,
\end{align*}
where in the first inequality we use the fact that $\widehat{\x}$ is an optimal extreme point solution of \eqref{lp:LP}, when the estimates are used in the objective. In the second inequality, we use that $|\widehat{p}_{i}(\tau) - p_{i}(\tau)| \leq \epsilon, \forall i \in \A, \tau \in \mathbb{N}$, while the third follows by constraint \eqref{lp:total} of \eqref{lp:LP}. 

We observe that \Cref{algo} can produce a feasible planning using any given extreme point solution of \eqref{lp:LP} (not necessarily the optimal).
Let $\widehat{A}_t$ be the set of arms played at some round $t \geq \tau^{\max}$ by \Cref{algo} (using the estimates) and let $\widehat{C}_t$ be the corresponding set of candidate arms. By construction of the algorithm, it is the case that 
\begin{align*}
    \widehat{A}_t = \underset{S \subseteq \widehat{C}_t, |S| \leq k}{\text{argmax}} \sum_{i \in S} \widehat{p}_i(\widehat{\sigma}_{i,t}), 
\end{align*}
where $\widehat{\sigma}_{i,t}$ is the actual delay of arm $i$ at time $t$ in a run of \Cref{algo} (using the estimates).

Let $\widehat{\tau}^*_i$ be the (sampled) critical delay of arm $i$ in a run of \Cref{algo} (using the estimates). By working along the lines of \Cref{thm:main1}, using monotonicity of $p_i(\tau)$ and the fact that $\widehat{\sigma}_{i,t} \geq \widehat{\tau}^*_i$, for the expected payoff collected at any round $t \geq \tau^{\max}$ (over the randomness of the payoff realizations), we have
\begin{align*}
    \max_{S \subseteq \widehat{C}_t, |S| \leq k} \sum_{i \in S} p_i(\widehat{\sigma}_{i,t}) \geq \max_{S \subseteq \widehat{C}_t, |S| \leq k} \sum_{i \in S} p_i(\widehat{\tau}^*_i)
    \geq \max_{S \subseteq \widehat{C}_t, |S| \leq k} \sum_{i \in S} \widehat{p}_i(\widehat{\tau}^*_{i}) - \epsilon \cdot k,
\end{align*}
where in the last inequality we use the fact that $\widehat{p}_i$ is $\epsilon$-close to $p_i$ for each arm $i \in \A$. 

Let $\gamma_k$ be the multiplicative approximation ratio of \Cref{algo} for $k$-\RS (given in \Cref{thm:main1}). By the analysis of \Cref{thm:main1}, it follows that 
$$
\Ex{}{\max_{S \subseteq \widehat{C}_t, |S| \leq k} \sum_{i \in S} \widehat{p}_i(\widehat{\tau}^*_{i})} \geq \gamma_k \cdot \widehat{V} \geq \gamma_k \cdot V^* - \gamma_k \cdot \epsilon \cdot k \geq \gamma_k \cdot \frac{\opt(T)}{T} - \gamma_k \cdot \epsilon \cdot k,
$$
where the last inequality follows by \Cref{lem:lpupperbound}.

Therefore, for any round $t \geq \tau^{\max}$, the expected payoff collected by \Cref{algo} (using the estimates) satisfies 
$$
\Ex{}{\max_{S \subseteq \widehat{C}_t, |S| \leq k} \sum_{i \in S} {p}_i(\widehat{\sigma}_{i,t})} \geq  \gamma_k \cdot \frac{\opt(T)}{T} - \mathcal{O} \left(\epsilon \cdot k \right).
$$
The proof follows by summing the above inequalities for every round $t \in [\tau^{\max}, T]$ and using linearity of expectation together with the fact that $\nicefrac{\opt(T)}{T} \leq k$. 
\end{proof}

\restateEstimation*
\begin{proof}
For any arm $i \in \A$ and delay $\tau$, let $\widehat{p}_i(\tau)$ be the empirical average of $m$ samples drawn independently from a distribution of mean $p_i(\tau)$ (and bounded in $[0,1]$). By Hoeffding inequality \cite{Hoeffding}, we have that
$$
\Pro{|\widehat{p}_i(\tau) - {p}_i(\tau)| > \epsilon} \leq 2 \exp\left(- 2 m \epsilon^2\right). 
$$
By setting $m = \frac{1}{2\epsilon^2} \ln(\frac{2\tau^{\max}n}{\delta})$, we get that $|\widehat{p}_i(\tau) - {p}_i(\tau)| \leq \epsilon$ with probability $1 - \frac{\delta}{\tau^{\max}n}$. 

Therefore, by a simple union bound over all arms $i \in \A$ and delays $\tau \in [\tau^{\max}]$, we get  
$$
\Pro{\exists i \in \A, \tau \in [\tau^{\max}] : |\widehat{p}_i(\tau) - {p}_i(\tau)| > \epsilon} \leq \delta,
$$
thus, conlcuding the proof. 
\end{proof}

\restateFactEstimation*
\begin{proof}
For any arm $i \in \A$ and delay $\tau$, we can collect $m$ independent samples from the distribution of mean $p_i(\tau)$, by playing arm $i$ for $m$ times -- once every $\tau$ rounds -- which can be achieved in $\mathcal{O}(m \tau)$ time steps. In order to collect $m$ samples for any delay $\tau \in [\tau^{\max}]$ for a fixed arm $i$, we thus need $\mathcal{O}\left(\sum_{\tau \in [\tau^{\max}]} m \tau\right) = \mathcal{O}\left(m (\tau^{\max})^2\right)$ rounds in the worst case.

Assuming that we can play at most one arm per round, in order to collect $m$ samples for each arm/delay pair, we need $\mathcal{O}(nm(\tau^{\max})^2)$ rounds in total. 

In a $k$-\RS instance, where we can play at most $k < n$ arms per round, we can parallelize and, thus, expedite the above sampling process by a factor of $k$, leading to a total of $\mathcal{O}(nm(\tau^{\max})^2/k)$ exploration rounds.
\end{proof}

\restateRegret*

\begin{proof}
Let $\text{R}(T)$ be the expected payoff collected by our ETC-based bandit algorithm in $T$ rounds. Recall that the $\gamma_k$-approximate regret for $T$ rounds is defined as 
$$
\text{Reg}_{\gamma_k}(T) = \gamma_k \opt(T) - \text{R}(T).
$$

For $\epsilon, \delta \in (0,1)$ (to be specified later), let $m = \frac{1}{2\epsilon^2} \ln\left(\nicefrac{2 \tau^{\max} n}{\delta}\right)$. By \Cref{fact:estimation}, we can collect $m$ samples from the distribution of each arm-delay pair in the first $\mathcal{O}\left(\frac{n (\tau^{\max})^2}{\epsilon^2 k}\ln(\frac{\tau^{\max} n}{\delta})\right)$ rounds. Since the mean payoffs are all bounded in $[0,1]$ and we can play at most $k$ arms per round, the regret accumulated for each of the exploration rounds is at most $k$, which implies that the total regret due to exploration can be upper bounded by $\mathcal{O}\left(\frac{n (\tau^{\max})^2}{\epsilon^2}\ln(\frac{\tau^{\max} n}{\delta})\right)$.

Moving on to the exploitation phase, for every arm $i \in \A$ and delay $\tau \in [\tau^{\max}]$, let $\widehat{p}_i(\tau)$ be the empirical estimate of $p_i(\tau)$ using $m$ samples collected in the exploration phase. For the rest $T - \mathcal{O}(\frac{n (\tau^{\max})^2}{\epsilon^2 k}\ln(\frac{\tau^{\max} n}{\delta})) \leq T$ rounds, by \Cref{lem:estimation}, we get that, with probability $1- \delta$, it holds $|\widehat{p}_i(\tau) - {p}_i(\tau)| \leq \epsilon$ for every arm $i$ and delay $\tau$. We refer to the above event as a ``nice'' sampling. 

Thus, the total regret accumulated in the exploitation phase in the case where sampling is not nice (which happens with probability $\delta$) is at most $T \cdot k \cdot \delta$. In the case where the sampling is nice, then \Cref{lem:robustness} suggests that for the remaining $T'$ rounds of the exploitation phase, we have
$$
R(T') \geq \gamma_k \opt(T') - \mathcal{O}(k \cdot \epsilon \cdot T' + k \cdot \tau^{\max}),
$$
which implies that the total regret accumulated in the exploitation phase can be upper bounded by $\mathcal{O}(k \cdot \epsilon \cdot T + k \cdot \tau^{\max} + T \cdot k \cdot \delta)$. By the above analysis, it follows that for any $\epsilon, \delta \in (0,1)$, the total regret of our bandit algorithm satisfies
\begin{align*}
\text{Reg}_{\gamma_k}(T) \leq \underbrace{\mathcal{O}\left(\frac{n (\tau^{\max})^2}{\epsilon^2}\ln(\frac{\tau^{\max} n}{\delta})\right)}_{\text{Exploration phase}}  + \underbrace{\mathcal{O}\bigg(k \cdot \epsilon \cdot T + k \cdot \tau^{\max} + T \cdot k \cdot \delta\bigg)}_{\text{Exploitation phase}}.
\end{align*}

By setting $\epsilon = \mathcal{O}\left(\sqrt[3]{\frac{n (\tau^{\max})^2 \ln(\tau^{\max} n T)}{k T}} \right)$ and $\delta = \frac{1}{T}$, we can upper bound the regret as
$$
\text{Reg}_{\gamma_k}(T) \leq \mathcal{O}\left(n^{1/3} \cdot (k \cdot \tau^{\max})^{2/3} \ln^{1/3}(\tau^{\max} n T) \cdot T^{2/3}  + k \cdot \tau^{\max}\right),
$$
thus proving that our algorithm achieves a sublinear (in the time horizon) regret guarantee.
\end{proof}

\end{document}